\documentclass{opt2022} 

\usepackage{amsmath}
\usepackage{amsthm}
\usepackage{amssymb}
\usepackage{bbm}
\usepackage{mathtools}
\usepackage{mathrsfs}
\usepackage{dsfont}

\usepackage{courier} 

\usepackage{lipsum} 
\usepackage{etoc}

\usepackage[space]{grffile} 
\usepackage{cancel} 
\usepackage{tablefootnote}

\usepackage{natbib}

\usepackage{hyperref}

\usepackage{wrapfig}
\usepackage{xcolor}
\definecolor{dark-red}{rgb}{0.4,0.15,0.15}
\definecolor{dark-blue}{rgb}{0,0,0.7}
\hypersetup{
    colorlinks, linkcolor={dark-blue},
    citecolor={dark-blue}, urlcolor={dark-blue}
}

\usepackage{algpseudocode}
\usepackage[vlined,linesnumbered,ruled,algo2e]{algorithm2e}
\let\oldnl\nl
\newcommand{\nonl}{\renewcommand{\nl}{\let\nl\oldnl}}

\usepackage{multicol}
\usepackage{enumitem}
\usepackage[utf8]{inputenc} 
\usepackage[T1]{fontenc}    
\usepackage{hyperref}       
\usepackage{booktabs}       
\usepackage{nicefrac}       
\usepackage{bm}

\usepackage{theoremref}
\mathtoolsset{showonlyrefs}

\newtheorem{thm}{Theorem}[]

\newtheorem{prop}{Property}[]

\usepackage{thmtools}
\usepackage{thm-restate}

\title[Optimization using Parallel Gradient Evaluations on Multiple Parameters]{Optimization using Parallel Gradient Evaluations \\ on Multiple Parameters}

\optauthor{%
\Name{Yash Chandak} \Email{ychandak@cs.umass.edu}\\
\addr University of Massachusetts
\AND
\Name{Shiv Shankar} \Email{sshankar@cs.umass.edu}\\
\addr University of Massachusetts
\AND
\Name{Venkata Gandikota} \Email{gandikota.venkata@gmail.com}\\
\addr Syracuse University
\AND
\Name{Philip Thomas} \Email{pthomas@cs.umass.edu}\\
\addr University of Massachusetts
\AND
\Name{Arya Mazumdar} \Email{arya@ucsd.edu}\\
\addr University of California San Diego
}

\begin{document}
\etocdepthtag.toc{mtchapter}
\etocsettagdepth{mtchapter}{subsection}
\etocsettagdepth{mtappendix}{none}

\maketitle

\begin{abstract}%
We propose a first-order method for convex optimization, where instead of being restricted to the gradient from a single parameter, gradients from multiple parameters can be used during each step of gradient descent. 
This setup is particularly useful when a few processors are available that can be used in parallel for optimization.
Our method
uses gradients from multiple parameters in synergy to update these parameters together towards the optima.
While doing so, it is ensured that the computational and memory complexity is of the same order as that of gradient descent.
Empirical results demonstrate that even using gradients from as low as \textit{two} parameters, our method can often obtain significant acceleration and provide robustness to hyper-parameter settings.
We remark that the primary goal of this work is less theoretical, and is instead aimed at exploring the understudied case of using multiple gradients during each step of optimization.
\end{abstract}


\section{Introduction}
Ranging from personal computers to internet-of-things enabled hardware, there exists a plethora of local computational devices that have capabilities of executing a few parallel processes but are otherwise limited in their memory and computational capacities. 
Having optimization procedures designed for such constrained systems can, therefore, enable wider use of machine learning models.
However, gradient descent based optimizers predominantly only use a single new gradient during each step of optimization.
Perhaps one common way to leverage multiple processes is by splitting the computation of the gradient for that \textit{single} parameter (e.g., by distributing data), if possible.
In this paper, we focus on a complementary problem, where we leverage multiple processes for gradient computation of \textit{multiple} parameters.
%

Particularly, we restrict our focus to optimization of convex functions.
Formally, for any convex function $f:\mathbb{R}^d \rightarrow \mathbb{R}$, 
and a fixed domain $\Omega \subset \mathbb{R}^d$ our goal is to find
\begin{align}
    \theta^* \in  \underset{\theta \in \Omega}{\text{arg min}} \,\, f(\theta).
\end{align}

Broadly, popular approaches for finding $\theta^*$ can be grouped into zeroth-order, first-order, or second-order methods.
In the following paragraphs, we briefly review how these approaches can potentially be used for the desired problem setup.  
Detailed discussions are deferred to Appendix \ref{related}.

\textbf{Zeroth-order methods:} Methods like genetic algorithms, evolutionary strategies, etc., work by evaluating the function $f(\theta)$ for multiple values of $\theta$ to estimate the optimization geometry and compute a direction that moves the parameters towards the optima \citep{davis1991handbook,hansen2015evolution}.
In our problem setup, each individual evaluation of $f(\theta)$ could ideally be done in parallel.
However, especially in high dimensions, these methods may scale poorly \citep{thierens1999scalability}, thereby requiring a lot of parallel resources, which are not often available.

\begin{figure*}
    \centering
    \includegraphics[width=0.65\textwidth]{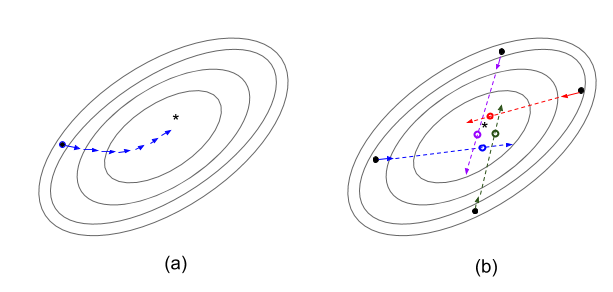}
    \caption{(a) Conventional gradient descent based methods have only one parameter vector which gets iteratively updated. (b) Proposed method maintains a set of $N$ parameters (here, $N=4$) and the gradients at each of those points are used to directly move the parameters near the minimizer of the function, marked as $\star$. Solid arrows depict the gradient and the proposed method uses them to jump to the points marked with empty circles in the direction of gradient. }
    \label{fig:idea}
    \vspace{-25pt}
\end{figure*}

\textbf{First-order methods:} Methods such as gradient descent, ADAM \citep{kingma2014adam}, etc., iteratively update a single parameter vector $\theta$ using the gradient $\nabla f(\theta)$ to move towards the optima. 
Multiple processes can be leveraged by performing independent runs from different initializations and taking the minimum among the runs.
However, over here, optimizing from a given initialization does not make use of the information available from other parallel searches. 
Thereby not making an ideal utilization of the already limited resource. 
Alternatively, often $f : \mathbb{R}^d \times \mathcal D \rightarrow \mathbb{R}$ is a function of data $\mathcal D$.
Multiple processes are often leveraged by distributing computation of $\nabla f(\theta;\mathcal D)$, for a common $\theta$, over mini-batches of data and then aggregating all gradients. 
This approach has the downside of not capturing information about the function at various points. 
In this work, we gain advantage over both the above mentioned approaches and show how $f(\theta; \mathcal D)$, for \textit{different} values of $\theta$ with mini-batches can be used to improve the search for $\theta^*$.

\textbf{Second-order methods: } Optimizers like Newton's method make direct use of the Hessian $\nabla^2 f(\theta)$ to move $\theta$ towards $\theta^*$ \citep{boyd2004convex}.
Access to the curvature of the function makes these methods more robust to hyper-parameters.
However, this requires inverting the Hessian, which can be computationally expensive for our setup.
Several quasi-Newton methods exist that minimize the computation required for estimating the Hessian inverse but still need to keep a matrix having a large $O(d^2)$ memory footprint \citep{nocedal2006quasi}.
Limited memory versions of quasi-Newton methods have also been developed that only need to store past $k$ \textit{sequentially} generated gradient evaluations \citep{liu1989limited}.

\textbf{Contribution Summary:} An ideal optimization procedure for our problem setup would meet the following criteria: 
%
    %
    \textbf{(a)} It should leverage the availability of the few parallel resources judiciously.
    %
    \textbf{(b)} It only uses first-order gradients to be computationally cheap. 
    %
    \textbf{(c)} It does not have large memory footprint, so that it can gracefully scale to large problems.
    %
    \textbf{(d)} Can provide quicker convergence to $\theta^*$, like the second-order methods.
    %
    \textbf{(e)} It is robust to hyper-parameters.

In this work, we propose a new optimization algorithm, called \textit{Gradient Grouping} to satisfy the above requirements.
The key idea behind the proposed algorithm is that it makes use of gradients from $N$ \textit{different} parameters in synergy, where $N$ can be as low as $2$.
During each update, the minima is estimated by extrapolating the $N$ gradients and finding their approximate point of intersection. 
All the parameters are then updated towards this point, and the process is repeated.


\section{Proposed Method}

In this section, we first motivate the underlying principle behind our approach and then derive an update rule.
Unlike existing methods that use only a single gradient from the current timestep or use multiple gradients obtained in the past, what if we had access to several gradients, \textit{from different parameters}, at any update step? 
In such a case, multiple gradients can provide additional information about the loss surface and can thus help in both acceleration and automatically  choosing step sizes.

The core idea is illustrated in Figure \ref{fig:idea}.
Given gradient evaluations from $N$ different points, we propose estimating the optima by extrapolating these gradients and finding their approximate point of intersection.
Formally, given $N$ parameter $\{\theta_i\}$, we propose to update each $\theta_i \in \mathbb{R}^{1\times D}$ as
\begin{align}
   \tilde \theta_i &= \theta_i + \eta_i \nabla f(\theta_i) &  \forall i \in [1, N],
   \label{eqn:gdupdate}
\end{align}
where $\eta_i$ is the corresponding step size and is set such that it directly makes the parameters $\theta_i$ jump to the (approximate) point of intersection of the gradient vectors $\{\nabla f(\theta_i) \}$.

In high dimensions, there is unlikely to be a unique point of intersection.
Therefore, we relax this exact intersection constraint and instead aim at making the parameters update to a point where they are closest to each other. 
Let $\Psi : \mathbb{R}^{d \times N} \rightarrow \mathbb{R}$ be the desired function of all the $\tilde \theta_i$ that is a (proxy) measure for getting towards the intersection point/area.
Our goal is to find,
\begin{align}
    \underset{\{\eta_1,...,\eta_N\}}{\text{arg min}} \,\, \Psi(\tilde \theta_1, ..., \tilde \theta_N),
\end{align}
and subsequently use the obtained $\eta_i$'s to update the respective parameters.
Some candidate choices for $\Psi$ can be a polytope/convex-hull of all the points of intersection, or the radius of the ball containing all the points of intersection.
However, it can also quite likely be that, in high dimensions, there is \textit{not even a single} point of intersection among the gradients.
To still ensure that after the update to the parameters, $\{\theta_i\}$, the resulting set of parameters, $\{\tilde \theta_i\}$, are as close to each other as possible, we define  $\Psi$ as the following, 
 \begin{align}
     \Psi(\tilde \theta_1, ..., \tilde \theta_N) &\coloneqq \sum_{i=1}^N \left \lVert \tilde \theta_i - \frac{\sum_j \tilde \theta_j}{N}\right\rVert^2. \label{eqn:obj}
 \end{align}

 Before proceeding further, we now establish few notations that will be used throughout the paper. 
 Let $\circ$ represent Hadamard (entrywise) product between two matrices of same dimensions defined as $(A\circ B)_{i,j} = A_{i,j} B_{i,j}$. 
 Let, $\Theta$ and $G$ corresponds to matrices with columns containing all the parameter vectors and their gradients, respectively.
Let all the step-sizes, $\eta_i$ for $i^{th}$ parameters, be in a vector $\eta$.
Let $L$ represents a Laplacian matrix for a fully connected graph with $N$ vertices and let $\mathbbm{1}$ denote a ones vector of size $N$.
That is,
    \begin{align}
        \Theta &\coloneqq [\theta_1, ..., \theta_N] & \in \mathbb{R}^{d\times N}, \hspace{60pt}
        G &\coloneqq [\nabla f(\theta_1), ..., \nabla f(\theta_N)] & \in \mathbb{R}^{d \times N}, \\
        \eta &\coloneqq [\eta_1, ..., \eta_N]^\top & \in \mathbb{R}^{N \times 1}, \hspace{60pt}
        \mathbbm{1} &\coloneqq [1, ..., 1]^\top & \in  \mathbb{R}^{N \times 1}, \\
        L &\coloneqq N\cdot I - \mathbbm{1}\mathbbm{1}^\top & \in \mathbb{R}^{N\times N}. \hspace{60pt}
    \end{align}

\begin{thm}
    For any given $N > 1$
    \begin{align}
        \underset{\eta}{\text{arg min}} \,\, \Psi(\tilde \theta_1, ..., \tilde \theta_N) &= - (G^\top G \circ L)^{-1} (G^\top\Theta \circ L) \mathbbm{1}.
    \end{align} 
    \thlabel{lemma:stepsize}
\end{thm}

\begin{proof}
For brevity, we denote $\Psi(\tilde \theta_1, ..., \tilde \theta_N)$ as $\Psi$.
We begin by expanding \eqref{eqn:obj},
\begin{align}
     \Psi &= \sum_{i=1}^N \left \lVert \tilde \theta_i - \frac{\sum_{i=1}^N \tilde \theta_i}{N}\right\rVert^2 
    \\
     &= \sum_{i=1}^N \left \lVert \theta_i + \eta_i \nabla f(\theta_i) - \frac{\sum_j (\theta_j + \eta_j \nabla f(\theta_j))}{N} \right \rVert^2 
     \\
     &= \sum_{i=1}^N \left |\! \left | \frac{N-1}{N}\theta_i - \frac{\sum_{j \neq i} \theta_j}{N}  + \eta_i \frac{N-1}{N}\nabla f(\theta_i) - \frac{\sum_{j \neq i} \eta_j \nabla f(\theta_j)}{N} \right |\! \right |^2. \label{eqn:non-matrix}
\end{align}

\noindent To solve for $\eta$ that minimizes \eqref{eqn:non-matrix}, we first convert \eqref{eqn:non-matrix} into matrix notations.
To illustrate the intermediate steps, we define the following,
\begin{align}
        M_i &\coloneqq  \left [-\frac{1}{N}, ..., \frac{N-1}{N}, ..., -\frac{1}{N} \right ] & \in \mathbb{R}^{N \times 1}, \\
        \Lambda_i &\coloneqq \text{diag}(M_i) & \in \mathbb{R}^{N \times N},
\end{align}
where $i^\text{th}$ position in $M_i$ is $(N-1)/N$  and the rest are $-1/N$.
%
Therefore, \eqref{eqn:non-matrix} can be expressed as,
\begin{align}
    \Psi = \sum_{i=1}^N \lVert \Theta M_i + G\Lambda_i \eta \rVert^2. \label{eqn:matrix}
\end{align}
Solving for $\eta$ that minimizes \eqref{eqn:matrix},
\begin{align}
    \frac{\partial \Psi}{\partial \eta} &= 2\sum_{i=1}^N ( \Theta M_i + G\Lambda_i \eta )^\top (G\Lambda_i ) = 2\sum_{i=1}^N M_i ^\top \Theta ^\top G\Lambda_i  + 2\sum_{i=1}^N \eta^\top \Lambda_i ^\top G^\top G\Lambda_i .
\end{align}
As $M_i = \Lambda_i \mathbbm{1}$, 
\begin{align}
    \frac{\partial \Psi}{\partial \eta} &= 2\sum_{i=1}^N \mathbbm{1}^\top \Lambda_i ^\top \Theta ^\top G\Lambda_i  + 2\sum_{i=1}^N \eta^\top \Lambda_i ^\top G^\top G\Lambda_i 
    = \mathbbm{1}^\top 2\sum_{i=1}^N  \Lambda_i ^\top \Theta ^\top G\Lambda_i  + \eta^\top 2\sum_{i=1}^N  \Lambda_i ^\top G^\top G\Lambda_i . \label{eqn:simplified}
\end{align}
Equating \eqref{eqn:simplified} to $0$ gives
$  \eta^\top \sum_{i=1}^N  \Lambda_i ^\top G^\top G\Lambda_i  = -\mathbbm{1}^\top \sum_{i=1}^N  \Lambda_i ^\top \Theta ^\top G\Lambda_i . 
 $ 
 Therefore,
  \begin{align}
  \eta^\top &= -\mathbbm{1}^\top \left(\sum_{i=1}^N  \Lambda_i ^\top \Theta ^\top G\Lambda_i \right) \left( \sum_{i=1}^N \Lambda_i ^\top G^\top G\Lambda_i \right)^{-1}. \label{eqn:summation}
\end{align}
To simplify \eqref{eqn:summation}, we make use of the following property (See Appendix \ref{apx:proofs} for a proof),
\begin{prop}
\label{prop:broadcast}
For any matrix $B \in \mathbb{R}^{N \times N}$, 
    $\sum_{i=1}^N  \Lambda_i^\top B \Lambda_i = \frac{1}{N} B \circ L.$
\end{prop}
%

\noindent Using Property \ref{prop:broadcast} in \eqref{eqn:summation}, 
\begin{align}
    \eta^\top &= -\mathbbm{1}^\top \left(\frac{1}{N}\Theta ^\top G \circ L \right) \left( \frac{1}{N}G^\top G \circ L \right)^{-1} = -\mathbbm{1}^\top \left(\Theta ^\top G \circ L \right) \left( G^\top G \circ L \right)^{-1},
\end{align}
Therefore,    $\eta = - (G^\top G \circ L)^{-1} \left(G^\top \Theta \circ L\right) \mathbbm{1}.$
\end{proof}

\section{Exploratory Insights}

\paragraph{One step convergence for quadratics with condition number $1$: } Perhaps interestingly, it can be shown that updating the parameters with step sizes  $\eta = - (G^\top G \circ L)^{-1} \left(G^\top \Theta \circ L\right) \mathbbm{1}$ converges to the optima in a \textbf{single} step for a quadratic function $f(\theta) \coloneqq \frac{1}{2}\theta^\top A \theta$, where $A$ is positive-semi-definite and has condition number $\rho =1$.
This result holds irrespective of the number of dimensions $d$ of $\theta \in \mathbb R^d$, and even with $N=2$.
We refer the readers to Appendix \ref{sec:isometric} for details.

\paragraph{Connections to approximate Newton's method:} The proposed approach uses gradient evaluations at various points on the loss surface to estimate the optima.
From one perspective, it can be seen as obtaining the curvature information of the loss surface and using it to obtain the step-sizes.
This is reminiscent of (quasi-) Newton's methods which use Hessian to estimate the loss surface and get the appropriate direction and the step-size for parameter update.
In Appendix \ref{apx:postcondition}, we discuss this connection is more details.

\section{Empirical Analysis}\label{sec:emp}

In this section, we present an empirical comparison on various convex optimization benchmarks between the proposed approach and popular optimization methods.
We call our proposed method Gradient Grouping (GG).
Exact algorithm is presented in Appendix \ref{sec:algo}.

\paragraph{Robustness to Hyper-parameter: }  We use several multi-class logistic regression datasets to evaluate the performance of our method, and the following baselines: stochastic gradient descent, Nesterov's accelerated gradient method \citep{boyd2004convex}, ADAM \citep{kingma2014adam}, Rmsprop \citep{tieleman2012lecture}, and a first order quasi-Newton method L-BFGS \citep{liu1989limited}.
Figure \ref{fig:performance} provides a comparison of the performance of our method  against performances of the baselines optimized with different values of learning rate.

The plots in Figure \ref{fig:performance} correspond to the most challenging setting for our method 
where \textbf{(a)} the number of available parallel resources is only $2$ and \textbf{(b)} while GG was developed for the deterministic gradient, it is typically infeasible to compute the full-batch gradient. Therefore, we use GG as-is in the setting where function evaluations are stochastic due to mini-batches.
Further, as the proposed method makes use of $2$ resources in parallel as opposed to the baselines that use $1$, we consider a strict setup for a fair comparison by making use of only half the data per computation of gradient.  This ensures that the proposed method neither uses more wall-time nor the total compute time.
Specifically, GG uses a batch size of $32$ for computing each gradient and the baselines use a batch-size of $64$. 

We notice that across all the domains, even without any hyper-parameter optimization, the proposed algorithm achieves nearly the same performance as that achieved by optimally tuned baselines.
This showcases the advantage of leveraging parallel resources to compute gradients for \textit{different} parameter vectors, even in the restricted case when only $2$ gradients can be obtained.
%

In comparison, note that although there exists a learning-rate with which the baselines are able to achieve the same performance, the baselines are particularly sensitive to it and their performance deteriorates significantly when it is not set to the optimal value.
Further, the optimal learning rates for these baselines vary across the domains, thereby necessitating a thorough hyper-parameter search. 

\paragraph{Ablation: Impact of N:} 
In problem setups, where more than $2$ parallel resources are available, a natural question to ask is how well does the proposed method perform then?
%
%
In Appendix 
\ref{apx:ablation} we present this ablation and showcase that while just $2$ parallel gradients sufficed to provide acceleration in high-dimensional problems, our method further benefits when more resources are available.

\begin{figure*}[t]
    \centering
    \includegraphics[width=0.30\textwidth]{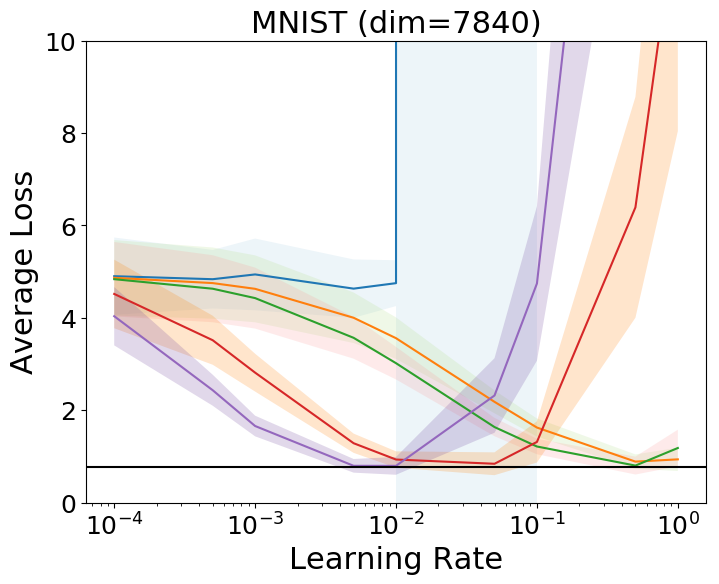}  \quad  \includegraphics[width=0.30\textwidth]{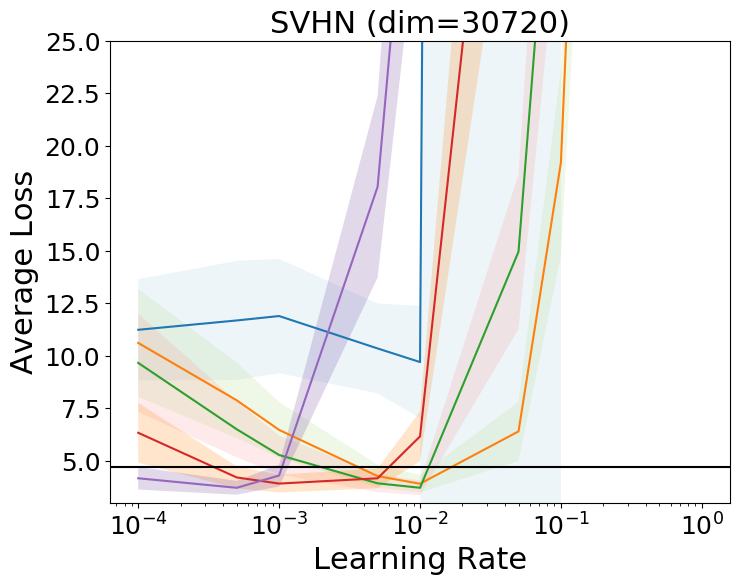} \quad
    \includegraphics[width=0.30\textwidth]{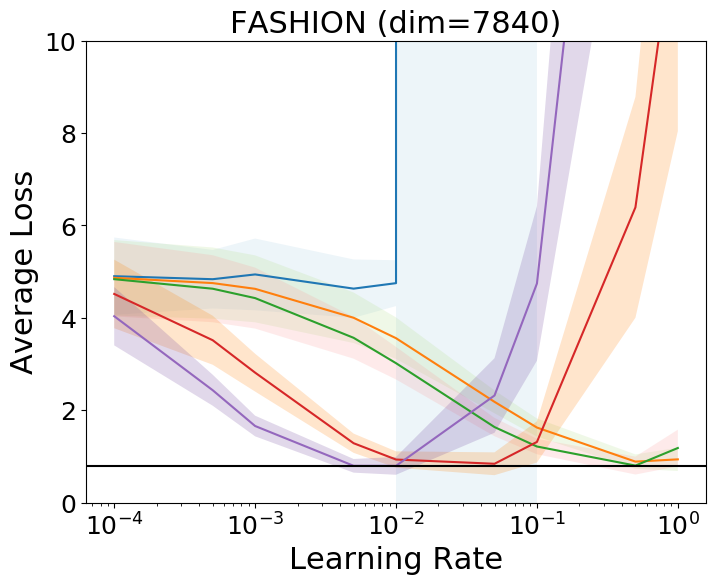}
\\
    \includegraphics[width=0.30\textwidth]{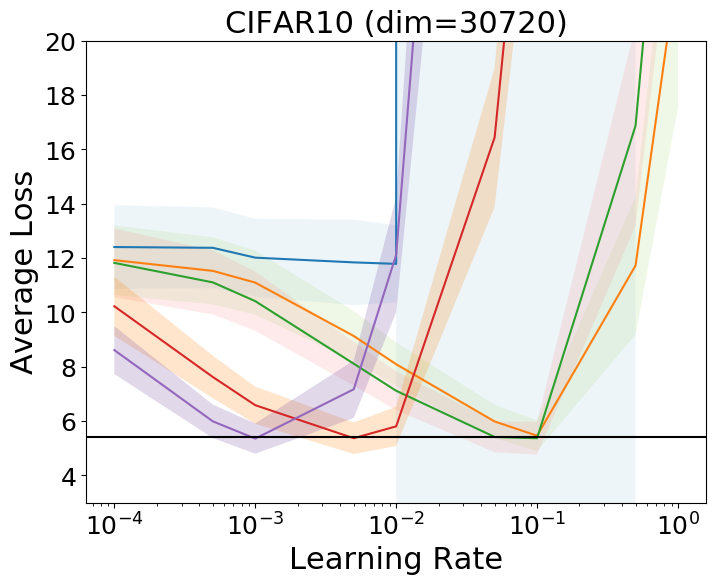} \quad
    \includegraphics[width=0.30\textwidth]{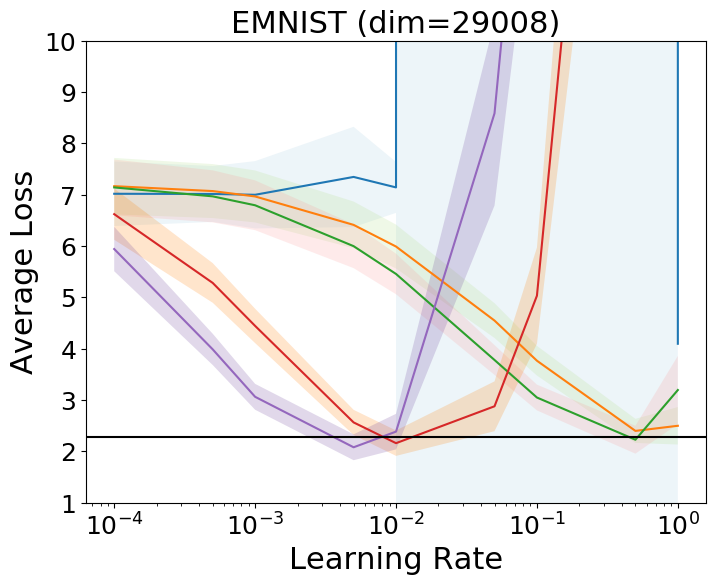} \quad
    \includegraphics[width=0.30\textwidth]{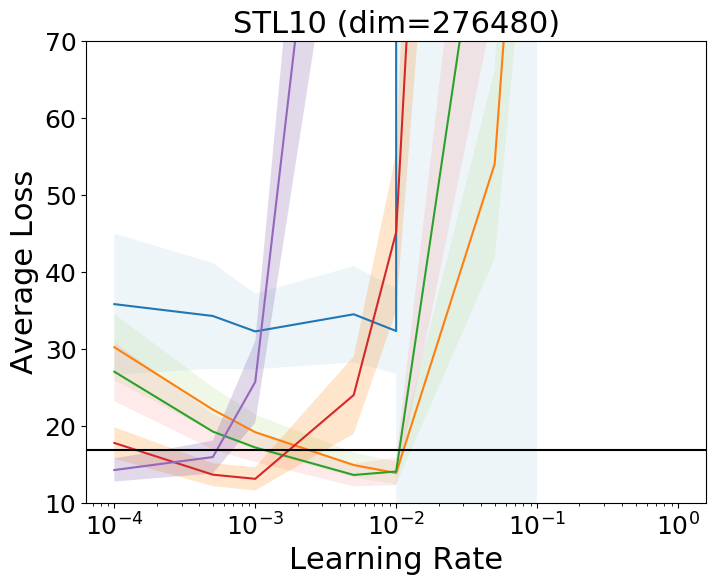}
    \\
    \includegraphics[width=0.80\textwidth]{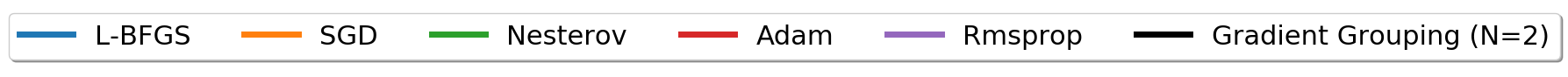}
    \caption{
    %
        Title of the sub-plots mention the name of the dataset and the dimension of parameters $\theta$ used. 
    The plots showcase sensitivity of the baselines to their learning rate hyper-parameter.
    Y-axis corresponds to the average loss over the entire training phase.
    The reference line (in black) corresponds to the performance of the proposed method obtained without performing any domain specific hyper-parameter tuning.
    Shaded regions correspond to standard deviation obtained using $10$ trials.
    }
    \label{fig:performance}
\end{figure*}

\section{Conclusion and Future Work}
We presented a convex optimization algorithm that exploits the availability of parallel resources by computing only the first-order gradients of different parameters in parallel and using them in synergy to reach the optima.
Through empirical results, we demonstrated that even with $2$ gradients computed in parallel, acceleration similar to the best tuned baselines can be achieved.
Further, using gradients in synergy achieves this accelerated performance without requiring any hyper-parameter tuning.
Enjoying both low memory and computationally cost, it can be used for quickly optimizing large problems on devices with limited compute resources.

An interesting future direction would be to efficiently combine the proposed method that makes use of gradients computed in parallel, with methods that do curvature correction by making use of past gradients that are computed sequentially.
Addressing this would allow bringing the best of both the paradigms. 
Another research question is to extend this approach to non-convex optimizations. The proposed approach relies upon extrapolating the gradients from $N$ different points. If these points happen to be in different `valleys', it is not clear how to optimally combine these gradients.
%


\bibliography{mybib}
\bibliographystyle{abbrvnat}

\clearpage
\appendix

\onecolumn

\setcounter{thm}{0}
\setcounter{cor}{0}
\setcounter{prop}{0}


\begin{center}
    \Large
    Supplementary Material
\end{center}

\etocdepthtag.toc{mtappendix}
\etocsettagdepth{mtchapter}{none}
\etocsettagdepth{mtappendix}{subsection}
\tableofcontents

\section{Detailed Related Work}
\label{related}

The literature on convex optimization is vast and no effort is enough to thoroughly review it.
Here, we briefly touch upon the popular paradigms in convex optimization that are related to our approach and discuss their differences from our method.

\textbf{Federated Learning:}     In many use-cases, due to either large amounts of data, or for accessing data only at a local server, distributed optimization is of prime importance.
    In such a scenario, data is partitioned into multiple sub-sets and for a given parameter, gradient computation with respect to each sub-set is computed in a decentralized fashion.
    The primary challenge over here includes how to aggregate these gradients effectively \citep{shi2015extra,zhang2014asynchronous,yang2019federated}.
    The proposed method presents a complementary procedure, where gradient with respect to different set of \textit{parameters} are used in synergy to update those parameters towards the optima. 

\textbf{Center of mass/gravity method:} In this method, the center of mass (CoM) of a valid search space is computed and then used to partition the valid search space for narrowing down the space consisting the minima.
    This process is iteratively carried out to reduce the domain of search and obtain the final solution \citep{bubeck2015convex}.
    This method is similar to ours as we also iteratively use gradients at various points in the domain to narrow down to the region consisting the optima.
    However, calculating the center of mass is computationally expensive and makes the CoM method dimension dependent.
    Therefore, unlike our approach, this method cannot be scaled to large problems.

\textbf{Genetic algorithms (GAs)/Evolutionary strategy (ES):}     Methods along this idea are for settings where computing the gradient is infeasible or computationally expensive.
    This method maintains a set of parameters, similar to our method, where the function value is evaluated.
    These evaluations are then used to compute an update direction for all the parameters in the set \citep{davis1991handbook}.
    In several cases, this procedure is also combined with gradients.
    For example, in hyper-parameter optimization (like architecture search for neural networks) function evaluation is itself an optimization procedure, where the quality of a hyper-parameter is estimated based on the quality of parameters obtained after performing gradient-descent using those hyper-parameters \citep{bergstra2011algorithms}.
    Another line of work first uses ES as an approximate global search procedure and then runs gradient descent for fine-tuning with a hope to more accurately estimate the global minima \citep{davis1991handbook}.
    To make such evolutionary strategies feasible in high-dimension space,  \citep{maheswaranathan2018guided} propose to only search in a meaningful sub-space of parameters that is obtained using recent history of (surrogate) gradients.
    These still requires a large population size of parameters.

    \textbf{Quasi-Newton methods:} Newton's method is a second-order optimization method which requires pre-conditioning the gradient with the inverse of the Hessian matrix.
    Due to this operation, they scale poorly to problems with large dimensions.
    Quasi-Newton methods (like SR1, LBFGS, etc.)  aim to use only first-order gradients to approximate Newton's method and trade-off computational complexity with convergence rates \citep{nocedal2006quasi}.
    While this is more scale-able than Newton's method, they still require pre-multiplication of the gradient with a $d \times d$ matrix, which can be expensive for large scale problems.
    L-BFGS further trades off performance quality by keeping only the past $k$ gradients that were obtained sequentially \citep{liu1989limited}.
    
    %
    
    \textbf{Conjugate gradients:} In this update procedure, for any $d$ dimensional space, $d$ orthogonal vectors are generated (iteratively) and then line search is used to minimize the objective in each of those directions \citep{bubeck2015convex}.
    Our method can be seen to be doing approximate line search by estimating the minima along the direction of gradient using $N$ separate gradient evaluations.
    This avoids the expensive computation required for both obtaining and iterating over $d$ orthogonal directions, and doing linear search each time.
    
    \textbf{Accelerated methods:} Accelerated optimization procedures like heavy-ball method, Nesterov's method, RMSprop, ADAM, AdaGrad, etc. leverage past gradients to accelerate search \citep{kingma2014adam,lydia2019adagrad}. 
    They are also related to acceleration procedures for ordinary differential equations (ODE), like Runge-Kutta \citep{dormand1980family} and Andersons acceleration \citep{walker2011anderson}.
    Optimizers built around these ODE techniques have also been developed \citep{scieur2016rna,zhang2018direct}. 
    Notice that all these methods use past gradients and function values that were obtained \textit{sequentially}.
    In contrast, our method uses multiple \textit{parallel} gradient evaluations at any update and is thus complementary to these approaches.
    For instance, update directions obtained using any of the above approaches can be substituted in place of $\nabla f(\theta)$ in our framework to leverage parallel processing.
    We leave such combinations for future work.

\section{Remaining Proofs}
\label{apx:proofs}
\begin{prop}
For any matrix $B \in \mathbb{R}^{N \times N}$, 
\begin{align}
    \sum_i  \Lambda_i^\top B \Lambda_i = \frac{1}{N} B \circ L.
\end{align}
\end{prop}
\begin{proof}

As $\Lambda_i$ is a diagonal matrix consisting of $M_i$ in the diagonal, $\Lambda B$ corresponds to row broadcast using $M_i$ and $B \Lambda$ corresponds to column broadcast using $M_i$.
Therefore,
\begin{align}
    \sum_i  \Lambda_i^\top B \Lambda_i &= 
    \sum_i  B \circ M_i M_i^\top \\
    &= B \circ \sum_i   M_i M_i^\top \label{eqn:laplacian1}.
\end{align}
In the above equations, $ M_i M_i^\top$ corresponds to a matrix  
where $(N-1)^2/N^2 $ is in the $(i,i)$ cell of the matrix, and both row $ i$ and column $i$ contain $-(N-1)/N^2$,
\[
 M_i M_i^\top = \begin{bmatrix}
1/N^2 & .. & -(N-1)/N^2 & ... & 1/N^2 \\
&&  ... && \\
-(N-1)/N^2 & .. & (N-1)^2/N^2 & ... & -(N-1)/N^2 \\
&&  ... && \\
1/N^2 & .. & -(N-1)/N^2 & ... & 1/N^2
\end{bmatrix},
\]

$\sum_i   M_i M_i^\top$ would imply that every off diagonal entry is a summation of $N-2$ terms of $1/N^2$ and two terms of $-(N-1)/N^2$.
This is because for any entry $(i,j)$, where $i \neq j$, it has $-(N-1)/N^2$ either when it is $(N-1)^2/N^2$ in the cell $(i,i)$ or in the cell $(j,j)$. 
Therefore, each off-diagonal element will sum to $(N-2)/N^2 - 2(N-1)/N^2 = - 1/N$.
Similarly, each diagonal entry will be a sum $N-1$ terms of $1/N^2$ and one term of $(N-1)^2/N^2$, which is equal to $(N-1)/N^2 + (N-1)^2/N^2 = (N-1)/N$.
Therefore,
\begin{align}
 \sum_i   M_i M_i^\top  = \frac{1}{N}\begin{bmatrix}
N-1 & -1 & ... & -1 \\
-1 & N-1 & ... & -1 \\
&&  ... & \\
-1 & -1 & ... & N-1 
\end{bmatrix} = \frac{1}{N}L. \label{eqn:laplacian2}
\end{align}
Using \eqref{eqn:laplacian2} in \eqref{eqn:laplacian1},
\begin{align}
    \sum_i  \Lambda_i^\top B \Lambda_i &= \frac{1}{N} B \circ L. 
\end{align}
\end{proof}

\section{Exploratory Insights}
In this section, we explore some properties of the proposed method for the fundamental case of quadratic functions, and draw some connections to prior methods.

\subsection{One step convergence for quadratics with condition number $\rho = 1$}
\label{sec:isometric}
Perhaps interestingly, it can be shown that even with $N=2$, GG can converge to the optima in a single step for a quadratic function $f(\theta) \coloneqq \frac{1}{2}\theta^\top A \theta$, where $A$ is positive-semi-definite and has condition number $\rho =1$.
This result holds irrespective of the number of dimensions $d$ of $\theta \in \mathbb R^d$.

Recall from \eqref{eqn:gdupdate}, the update rule in closed form can be expressed as,
\begin{align}
   \Theta_{t+1} \leftarrow \Theta_t + G \texttt{diag}(\eta), \label{eqn:updatee}
\end{align}
where for the proposed GG,  
\begin{align}
    \eta = - (G^\top G \circ L)^{-1} \left(G^\top \Theta \circ L\right) \mathbbm{1}.
\end{align}
Further, as $\nabla f(\theta) = A\theta$, $\eta$ can be expressed as,
\begin{align}
    \eta &= - (G^\top G \circ L)^{-1} \left(G^\top \Theta \circ L\right) \mathbbm{1} 
    %
    = - (\Theta^\top A^\top A \Theta \circ L)^{-1} \left(\Theta^\top A^\top \Theta \circ L\right) \mathbbm{1}. 
    \label{eqn:qw}
\end{align}
Now notice that as the condition number of $A$ is one, there exists some $c$ such that  $A = cI$.
To see this, For a quadratic function we can presume A to be symmetric because \begin{align}
    \theta^\top A \theta = \frac{1}{2}\theta^\top (A + A^\top)\theta,
\end{align}
and so we can replace with $A$ with a symmetric matrix $\bar A \coloneqq \frac{1}{2}(A +A^\top)$.
Furthermore, convexity implies that the Hessian of $f(\theta)$ must be PSD, therefore $A$ must be PSD. 

Now as real-symmetric matrices are diagonalizable, using eigenvalue decomposition we get 
\begin{align}
    A=Q\Lambda Q^\top,
\end{align}  
where $Q$ is an orthogonal matrix and $\Lambda$ is a diagonal matrix consisting of the eigenvalues of $A$. Because the PSD condition ensures that all eigenvalues are nonnegative, and the condition number is 1, $(\text{max eigen value of } A)/((\text{min eigen value of } A))=1$. This implies that $\Lambda=cI$ for some constant $c$. Therefore, 
\begin{align}
    A = Q\Lambda Q^\top = Q(cI)Q^\top = cQQ^\top = cI.
\end{align}

Therefore, \eqref{eqn:qw} can be simplified as the following,
\begin{align}
    \eta &= - (c^2 \Theta^\top I^\top I \Theta \circ L)^{-1} \left(c\Theta^\top I^\top \Theta \circ L\right) \mathbbm{1} 
    %
    = - \frac{1}{c}( \Theta^\top \Theta \circ L)^{-1} \left(\Theta^\top \Theta \circ L\right) \mathbbm{1} 
    = - \frac{1}{c} \mathbbm{1}, \label{eqn:eta}
\end{align}
when the inverse exists. Here the inverse may not exist when the two gradients are exactly parallel (for e.g., the parameters in $\Theta$ are exactly the same).  While this is unlikely in high dimensions, in Section \ref{sec:emp} we discuss a simple practical trick to avoid this.
Now, substituting the value of $\eta$ from \eqref{eqn:eta} in the update step, 
\begin{align}
    \Theta_{t+1} &= \Theta_t + G \texttt{diag}(\eta)
    = \Theta_t - A \Theta_t \texttt{diag}(\frac{1}{c} \mathbbm{1}). \label{eqn:tempp} 
    \end{align}
 Finally, the result can be obtained by further simplifying \eqref{eqn:tempp},
    \begin{align}
       \Theta_{t+1} &= \Theta_t - c I \Theta_t \frac{1}{c} = \Theta_t - \Theta_t = 0.
\end{align}
As the minima of $f(\theta)$ is at $\theta=0$, the update procedure finds the minima in \textbf{one} update (two gradient calls when $N=2$).
This is in contrast with the existing gradient based optimizers that can solve this in one step \textit{if and only if} the step-size is chosen appropriately.
Further, notice that the above result can be generalized for any $N>1$, and also other functions $f$ with condition number $\rho=1$.
%
%

%

%

\subsubsection{For setting with $\rho \neq 1$}
\label{apx:asymanalysis}

The goal of this section is to understand the behavior of gradient grouping for the fundamental case of quadratics when $\rho \neq 1$.
Particularly, in the following paragraphs we provide analysis using Eigen decomposition for one step update using the gradient grouping algorithm with $N=2$.
%
%
We consider $f(\theta)\coloneqq \frac{1}{2}\theta^\top A\theta$, where $A$ is a symmetric positive definite matrix.
As the optima of $f$ is at $0$, by looking at the norm of the parameters in $\Theta$ we can analyze the improvement towards the optima.
%
%

To do so, we first expand $\eta$'s formula for $N=2$.
For brevity, we use $g_1$ and $g_2$ to represent the gradients at the two points $\theta_1$ and $\theta_2$.
From \thref{lemma:stepsize},
\begin{align}
    \eta &= - (G^\top G \circ L)^{-1} \left(G^\top \Theta \circ L\right) \mathbbm{1} 
    \\
    &= - 
    \begin{bmatrix}
    g_1^\top g_1 & - g_1 ^\top g_2 \\
    - g_2^\top g_1 &  g_2 ^\top g_2 
    \end{bmatrix}^{-1}
    \begin{bmatrix}
    g_1^\top \theta_1 & - g_1 ^\top \theta_2 \\
    - g_2^\top \theta_1 &  g_2 ^\top \theta_2 
    \end{bmatrix} \mathbbm{1} 
    \\
    &\overset{(a)}{=} - \frac{1}{D} 
    \begin{bmatrix}
    g_2^\top g_2 &  g_1 ^\top g_2 \\
     g_2^\top g_1 &  g_1 ^\top g_1 
    \end{bmatrix}
    \begin{bmatrix}
    g_1^\top (\theta_1  - \theta_2) \\
    - g_2^\top (\theta_1  -  \theta_2) 
    \end{bmatrix} 
    \\
    &= - \frac{1}{D} 
    \begin{bmatrix}
    (g_2^\top g_2)(g_1^\top (\theta_1 - \theta_2)) -  (g_1^\top g_2)( g_2^\top (\theta_1  - \theta_2)) \\
     (g_2^\top g_1) (g_1^\top (\theta_1  - \theta_2)) -  (g_1 ^\top g_1) (g_2^\top (\theta_1  -  \theta_2)) 
    \end{bmatrix}
    \\
    &= - \frac{1}{D} 
    \begin{bmatrix}
    ((g_2^\top g_2)g_1^\top - (g_1^\top g_2)g_2^\top) (\theta_1 - \theta_2)  \\
     ((g_2^\top g_1)g_1^\top -  (g_1 ^\top g_1) g_2^\top )(\theta_1  - \theta_2) 
    \end{bmatrix}
    \\
    &= - \frac{1}{D} 
    \begin{bmatrix}
    ((g_2^\top g_2)g_1 - (g_1^\top g_2)g_2)^\top A^{-1}(g_1 - g_2)  \\
     ((g_2^\top g_1)g_1 -  (g_1 ^\top g_1) g_2 )^\top A^{-1}(g_1  - g_2) 
    \end{bmatrix}
    \\
    &\overset{(b)}{=} - \frac{1}{D} 
    \begin{bmatrix}
    ((\theta_2^\top A^\top A \theta_2)\theta_1^\top A^\top - (\theta_1^\top A^\top A \theta_2)\theta_2^\top A^\top) (\theta_1 - \theta_2)  \\
     ((\theta_2^\top A^\top A \theta_1)\theta_1^\top A^\top -  (\theta_1 ^\top A^\top A \theta_1) \theta_2^\top A^\top )(\theta_1  - \theta_2) 
    \end{bmatrix}
    \\
    &= - \frac{1}{D} 
    \begin{bmatrix}
    ((\theta_2^\top A^2 \theta_2)\theta_1 - (\theta_1^\top A^2 \theta_2)\theta_2 )^\top A^\top (\theta_1 - \theta_2)  \\
     ((\theta_2^\top A^2 \theta_1)\theta_1  -  (\theta_1 ^\top A^2 \theta_1) \theta_2)^\top A^\top (\theta_1  - \theta_2) 
    \end{bmatrix} \label{eqn:expanded},
\end{align}
where in (a) the notation $D$ represents determinant when computing the inverse, and (b) follows from expanding the gradient of $f(\theta)$ and using the fact that $A$ is a symmetric matrix.

As the labeling of parameters are arbitrary, without loss of generality let us consider only the first row (step-size for updating $\theta_1$) to analyze one (of the two) parameters in the set.
%
%
Using \eqref{eqn:gdupdate},
\begin{align}
\widetilde{\theta_1} = \theta_1 - \dfrac{1}{D}((\theta_2^{\top}A^2\theta_2)\theta_1 - (\theta_1A^2\theta_2)\theta_2)^{\top}A(\theta_1 - \theta_2)A\theta_1),
\end{align}
where we have used $\widetilde {\theta_1}$ to denote the value of $\theta_1$ after one update.
Now evaluating the norm of $\widetilde{\theta_1}$,
%
\begin{align}
||\widetilde{\theta_1}|| &= ||\theta_1 - \dfrac{1}{D}((\theta_2^{\top}A^2\theta_2)\theta_1 - (\theta_1A^2\theta_2)\theta_2)^{\top}A(\theta_1 - \theta_2)A\theta_1)|| \\
&=  \dfrac{1}{D}|| (D)\theta_1 - ((\theta_2^{\top}A^2\theta_2)\theta_1 - (\theta_1A^2\theta_2)\theta_2)^{\top}A(\theta_1 - \theta_2)A\theta_1)|| \\
&=  \dfrac{1}{D}|| ((\theta_1^{\top}A^2\theta_1)(\theta_2^{\top}A^2\theta_2) - (\theta_1^{\top}A^2\theta_2)^2)\theta_1 - ((\theta_2^{\top}A^2\theta_2)\theta_1 - (\theta_1A^2\theta_2)\theta_2)^{\top}A(\theta_1 - \theta_2)A\theta_1)|| 
\\
&=  \dfrac{1}{D}|| {\color{red}(\theta_1^{\top}A^2\theta_1)(\theta_2^{\top}A^2\theta_2)\theta_1} - {\color{blue}(\theta_1^{\top}A^2\theta_2)^2\theta_1 }- 
{\color{red}(\theta_2^{\top}A^2\theta_2)(\theta_1^\top A \theta_1)  A \theta_1}
\\
& \quad\quad\quad 
+ (\theta_2^{\top}A^2\theta_2)(\theta_1^\top A \theta_2) A \theta_1 
+ {\color{blue}(\theta_1A^2\theta_2)(\theta_2^\top A \theta_1) A \theta_1} 
- (\theta_1A^2\theta_2)(\theta_2^\top A \theta_2) A \theta_1 ||
%
\end{align}
%
To simplify, we now regroup terms with the matching colors,
\begin{align}
||\widetilde{\theta_1}|| 
%
&=  \dfrac{1}{D}||( {\color{red}(\theta_2^{\top}A^2\theta_2)\underbrace{((\theta_1^{\top}A^2\theta_1)I - (\theta_1^{\top}A\theta_1)A)\theta_1}_{\text{\normalfont S1}}} 
+ {\color{blue} (\theta_1^{\top}A^2\theta_2)\underbrace{((\theta_2^{\top}A\theta_1)A - (\theta_1^{\top}A^2\theta_2)I)\theta_1}_{\text{\normalfont S2}}} \\ 
& \quad \underbrace{((\theta_2^{\top}A^2\theta_2)(\theta_1^{\top}A\theta_2) - (\theta_1^{\top}A^2\theta_2)(\theta_2^{\top}A\theta_2) )}_{\text{\normalfont S3}}A\theta_1)||. \label{eqn:b0}
\end{align}

Now we bound the behavior of S1, S2, and S3 in terms of the eigenvalues of $A$. 
For the following, we denote the Eigenvalues of $A$ as $\lambda_i$ and the corresponding Eigenvector $v_i$. Note that as $A$ is positive definite, $\lambda_i > 0$. We decompose both $\theta_1,\theta_2$ in the Eigenbasis and call the coefficients associated with each basis vector as $\alpha_i$, and $\beta_i$, respectively. That is, $\theta_1 = \sum_i \alpha_iv_i$ and $\theta_2 = \sum_i \beta_iv_i$.
Under this expansion we get the following useful expressions:
\begin{align}
    &\theta_1^{\top}A\theta_1 = (\sum_i \alpha_iv_i)^{\top}A(\sum_i \alpha_iv_i) = (\sum_i \alpha_iv_i)^{\top}(\sum_i \lambda_i\alpha_iv_i) = \sum_i\lambda_i\alpha_i^2 \\
    &\theta_1^{\top}A\theta_2 = (\sum_i \alpha_iv_i)^{\top}A(\sum_i \beta_iv_i) = (\sum_i \alpha_iv_i)^{\top}(\sum_i \lambda_i\beta_iv_i) = \sum_i\lambda_i\alpha_i\beta_i \\
    &\theta_1^{\top}A^2\theta_2 = (\sum_i \alpha_iv_i)^{\top}A^2(\sum_i \beta_iv_i) = (\sum_i \alpha_iv_i)^{\top}(\sum_i \lambda_i^2\beta_iv_i) = \sum_i\lambda_i^2\alpha_i\beta_i\\
    &\theta_1^{\top}A^2\theta_1 = (\sum_i \alpha_iv_i)^{\top}A^2(\sum_i \alpha_iv_i) = (\sum_i \alpha_iv_i)^{\top}(\sum_i \lambda_i^2\alpha_iv_i) = \sum_i\lambda_i^2\alpha_i^2.
\end{align}

Plugging in these expressions to expand S1,S2 and S3,
\begin{align}
   ||\text{\normalfont S1}|| &= ||((\theta_1^{\top}A^2\theta_1)I - (\theta_1^{\top}A\theta_1)A)\theta_1|| \\
   &= ||(\sum_i \alpha_i^2\lambda_i^2)(\sum_i \alpha_i v_i) - (\sum_i \alpha_i^2\lambda_i)(\sum_i \alpha_i\lambda_iv_i)|| \\
   &= ||\sum_i\sum_j ( \alpha_i^2\alpha_j\lambda_i^2 - \alpha_i^2\alpha_j\lambda_i\lambda_j)v_j||\\
   &= ||\sum_i\sum_j  \alpha_i^2\alpha_j\lambda_i^2 (1-\frac{\lambda_j}{\lambda_i})v_j|| \\
   &= ||\sum_i\sum_j \alpha_i^2\alpha_j\lambda_i^2 (\frac{\lambda_j}{\lambda_i}-1)v_j|| \\
   &\leq ||\sum_i\sum_j  \alpha_i^2\alpha_j\lambda_i^2 (\rho-1)v_j|| \\
   &= ||\sum_i\sum_j  \alpha_i^2\alpha_j\lambda_i^2 v_j||(\rho - 1), \label{eqn:b1}
\end{align}

where $\rho \coloneqq \frac{\lambda_{\max}}{\lambda_{\min}}$. Similarly,

\begin{align}
   ||\text{\normalfont S2}|| &= ||((\theta_2^{\top}A\theta_1)A - (\theta_1^{\top}A^2\theta_2)I)\theta_1|| \\
   &= ||(\sum_i \alpha_i\beta_i\lambda_i)(\sum_i \alpha_i\lambda_iv_i) - (\sum_i \alpha_i\beta_i\lambda_i^2)(\sum_i \alpha_iv_i)|| \\
   &= ||\sum_i\sum_j  \beta_i\alpha_i\alpha_j \lambda_i^2 (\frac{\lambda_j}{\lambda_i} - 1)v_j ||\\
   &\leq ||\sum_i\sum_j  \beta_i\alpha_i\alpha_j \lambda_i^2 v_j||(\rho - 1). \label{eqn:b2}
\end{align}
Similarly,
\begin{align}
   ||\text{\normalfont S3}|| &= ||(\theta_2^{\top}A^2\theta_2)(\theta_1^{\top}A\theta_2) - (\theta_1^{\top}A^2\theta_2)(\theta_2^{\top}A\theta_2) || \\ 
   &= ||(\sum_i \beta_i^2\lambda_i^2)(\sum_i \beta_i\alpha_i\lambda_i) - (\sum_i \alpha_i\beta_i\lambda_i^2)(\sum_i \beta_i^2\lambda_i)|| \\
   &= ||\sum_i\sum_j  \beta_i^2\beta_j\alpha_j(\lambda_i^2\lambda_j -\lambda_i\lambda^2_j) ||\\
   &= ||\sum_i\sum_j  \beta_i^2\beta_j\alpha_j\lambda_i^2\lambda_j(1-\frac{\lambda_j}{\lambda_i}) ||\\
   &\leq ||\sum_i\sum_j  \beta_i^2\beta_j\alpha_j\lambda_i^2\lambda_j||(\rho - 1). \label{eqn:b3}
\end{align}
Using \eqref{eqn:b1}, \eqref{eqn:b2}, and \eqref{eqn:b3} in \eqref{eqn:b0}, observe that distance of $||\widetilde \theta_1||$ from the optimum can be bounded in terms proportional to $(\rho -1)$.
%
Notice that similar to discussions in \ref{sec:isometric}, this also shows that if $\rho = 1$ we reach the optimum in 1 step, and that if $\rho = 1 + \delta$, for some small $\delta$, the parameters can approach the optima quickly. 
Extending this analysis to characterize the evolution of parameters for the entire sequence of updates remain an interesting future direction.

\subsection{Gradient post-conditioning (instead of pre-conditioning)}

The proposed Gradient Grouping approach uses gradient evaluations at various points on the loss surface to estimate the optima.
From one perspective, it can be seen as obtaining the curvature information of the loss surface and using it to obtain the step-sizes.
This is reminiscent of (quasi-) Newton's methods which use Hessian to estimate the loss surface and get the appropriate direction and the step-size for parameter update.

\label{apx:postcondition}

Given a quadratic of the form $f(\theta) \coloneqq \frac{1}{2}\theta^\top A\theta$, the first and second order derivatives are $\nabla f(\theta) = A\theta$, and $\nabla^2 f(\theta) = A$, respectively.
According to the Newton method's update rule,
$$\theta_{t+1} \leftarrow \theta_{t} - A^{-1}g_{t}.$$
However, direct pre-multiplication with $A^{-1}$ requires both matrix inversion and a $d\times d$ matrix-vector multiplication, which can be computationally expensive. 
Several prior methods have studied quasi-Newton methods to minimize the computational requirements \citep{boyd2004convex}. 

In the following, we discuss a complementary method that draws some connections to the proposed GG method.
Let the Newton method's update for a set of points $\Theta \in \mathbb{R}^{d \times N}$ using $G \in \mathbb{R}^{d \times N}$ be,
\begin{align}
    \Theta_{t+1} \leftarrow \Theta_{t} - A^{-1}G_{t}, \label{eqn:qw12}
\end{align}
where $A^{-1}G_{t} \in \mathbb{R}^{d \times d} \times \mathbb{R}^{d \times N}.$
To reduce the computational cost, instead of pre-multiplying with $A^{-1}$, we aim to post multiply $G$ with another matrix $B \in \mathbb{R}^{N \times N}$, such that,
$$A^{-1}G_{t} \approx G_{t}B.$$
To do so, we can estimate $B$ such that,
$$B = \underset{B}{\text{argmin}} \lVert A^{-1}G_{t} - G_{t}B \rVert^2. $$
Solving for the above equation,
\begin{align}
    \frac{ \partial \lVert A^{-1}G_{t} - G_{t}B \rVert^2}{\partial B} &= G_t^\top(A^{-1}G_{t} - G_{t}B) = G_t^\top A^{-1}G_{t} - G_t^\top G_{t}B.
\end{align}
Equating the derivative to $0$, 
\begin{align}
     G_t^\top G_{t}B &= G_t^\top A^{-1}G_{t},
     \\
     B &= (G_t^\top G_{t})^{-1} (G_t^\top A^{-1}G_{t}). \label{eqn:shiv1}
\end{align}
Now as $\nabla f(\theta) = A\theta$,
we know
    $A^{-1} \nabla f(\theta)  = \theta. \label{eqn:shiv2}$
%
Substituting this in \eqref{eqn:shiv1}, 
\begin{align}
    B &= (G_t^\top G_{t})^{-1} (G_t^\top \Theta_t). \label{eqn:shiv3}
\end{align}

Perhaps interestingly, $B$ in \eqref{eqn:shiv3} shares a striking resemblance with $\eta$ from \thref{lemma:stepsize}.
Particularly, as $\Theta_t$ is updated by $A^{-1}G_t \approx G_t B$ in \eqref{eqn:qw12}, $B$ plays a role similar to $\eta$ in Line 9 of Algorithm \ref{apx:Alg:1}. 
Nonetheless, there are important differences. For instance, $B$ does not have the Laplacian and the product with the ones vector. This makes $B$ a full-matrix, unlike $\eta$.

Notice that while \eqref{eqn:shiv3} is useful to draw some connections with GG, it is not practically useful.
Unlike how $\eta$ was obtained from \thref{lemma:stepsize}, $B$ was obtained by enforcing a specific quadratic structure, which is crucial for the above steps. 
For instance, if $f(\theta) \coloneqq \theta^\top A \theta + C^\top\theta$, where $C$ is some vector, then $B$ will depend on additional terms dependent on $C$ which we do not typically have access to.
%


\section{Algorithm}\label{sec:algo}
In this section, we discuss the proposed algorithm.
As we maintain a group of parameters that compute gradients in parallel and use it to help each other obtain a common goal,
we call our approach \textit{Gradient Grouping}.
A pseudo code is presented in Algorithm \ref{apx:Alg:1}.
	\IncMargin{1em}
	\begin{algorithm2e}
	\caption{Gradient Grouping}\label{apx:Alg:1} 
		\textbf{Input} $N \in \mathbb{N}_{\geq 2}$  \Comment{Parameter set size} \\
		\textbf{Randomly Initialize} $\Theta \in \mathbb{R}^{d\times N}$  		\Comment{Parameter set}\\
		\textbf{Zero Initialize} $G \in \mathbb{R}^{d\times N}$  		\Comment{Gradient set}\\
		$L = N\cdot I - \mathbbm{1}\mathbbm{1}^\top$ \\
		\For{$t = 1,2,3,...$}{
			\For {$idx = 1,2,..., N$ (in parallel)}{
			$G_t[:, \,idx] = \nabla f(\Theta_t[:, \,idx])$  \Comment{Compute gradients}
    		}
    		
    		\vspace{5pt}
    		$\eta_t = - ( {G_t}^\top G_t \circ L)^{-1} \left( {G_t}^\top \Theta_t \circ L\right) \mathbbm{1}.$ \Comment{Lemma 1.}
    		\\
    		$\Theta_{t+1} \leftarrow \Theta_t + G_t \texttt{diag}(\eta_t)$ \Comment{Update parameters}
		}
\end{algorithm2e}
	\DecMargin{1em}

In Line 1, the number of parameters, $N$, to be considered in parallel is taken as the input.
In Lines 2 and 3, the parameter matrix and the gradient matrix are initialized.
As the Laplacian matrix, $L$, will always be fixed during updates, Line 4 creates and stores it beforehand.
Lines 5 to 9 correspond to the main optimization loop.
To avoid overloading the notation, we use super-script $t$ to denote the values in the $t^{th}$ iteration of optimization. 
In Lines 6 and 7, gradients are evaluated for all the $N$ parameters in parallel.
Using Lemma \ref{lemma:stepsize}, the step-sizes are obtained in Line 8.
Finally, in Line 9, the parameters are updated using the gradients and the step sizes obtained in the Line 8.

\subsection{Practical Modifications} Notice that the inverse of $(G^\top G \circ L)$ may not exist if the gradients are parallel (for e.g., $\Theta$ contains exactly the same parameters, or if the parameters are different but the gradients are facing each other, etc. )
To ensure that the inverse always exists in practice, we start all the parameters from a different initialization.
To make this condition also holds during the learning process, we multiply $\eta$ with a scaling factor $0 <\alpha < 1$, such that, $\eta \leftarrow \alpha \eta$.
This ensures that even when the extrapolated gradients are intersecting, the parameters will never be at the same point after an update.
More generally, we ensure invertibility  of $(G^\top G \circ L)$ by clipping its minimum Eigen value to a small positive constant $\epsilon$.
As the $(G^\top G \circ L)$ matrix belongs to $\mathbb{R}^{N \times N}$, the computation cost of this operation is negligible.
%
%
%
Specifically, leveraging the Eigen value decomposition, let 
$$G^\top G \circ L = V^\top \Lambda V.$$ 
Note that $G^\top G$ is PSD and it is known that a Laplacian matrix $L$ is also PSD.
Now using Schur's theorem, which asserts that the Hadamard product of two PSD matrices is also a PSD matrix, it can be established that $G^\top G \circ L$ is also a PSD matrix \citep[Theorem 3.4]{million2007hadamard}, \citep{schur1911bemerkungen,enwiki:1016337944}.
To convert $G^\top G \circ L$ into a PD matrix, we create a $\Lambda ^+ \coloneqq \texttt{clip}(\Lambda, \text{min}=\epsilon, \text{max}=\infty)$ and use $V^\top \Lambda ^+ V$ instead of $G^\top G \circ L$ every time.
As the $G^\top G \circ L$ matrix belongs to $\mathbb{R}^{N \times N}$, where $N < 10$ for all our experiments, the computation cost of this operation is negligible.
%

The two practical modifications to our method introduced the hyper-parameters: $\alpha$ and $\epsilon$.
In our experiments, we keep $\alpha = 0.9$ and $\epsilon = 1e-4$ throughout and perform no problem specific hyper-parameter tuning.

\subsection{Computational and Memory Complexity}
An important property of our method is that it is cheap in both computational and memory requirements.
Notice that $(G^\top G \circ L)$ is a $\mathbb{R}^{N \times N}$ matrix, irrespective of the number of dimensions $d$ of the parameters space.
As  $N$ can be as low as 2, and in general $N \ll d$, computational cost of inverting $(G^\top G \circ L)$ is negligible.
Apart from that, as each processor only stores one copy of the parameter $\theta$ and its corresponding gradient $\nabla f(\theta)$, its memory and computational cost is in the same order as that of gradient descent.
Further, the proposed method leverages curvature information by  using the $N$ individual parameters and their respective first-order gradients, and thus second-order derivatives, $\nabla^2 f(\theta)$, are never required to be computed.

\section{Experimental Details}

Implementations of the baseline algorithms: L-BFGS, SGD, Nesterov's method, Adam, and Rmsprop were based on the default routines available in PyTorch \citep{paszke2017automatic}.
For the experiments, convexity was ensured using a single layer neural network implementation in PyTorch \citep{paszke2017automatic} with softmax classification loss. 

The following datasets were used to report the empirical results.

\begin{table}[h]
\begin{minipage}{0.8\textwidth}
\begin{center}
    
\begin{tabular}{c|c}
     Dataset & Input dimensions\\
     \hline
     CIFAR\tablefootnote{https://www.cs.toronto.edu/~kriz/cifar.html} &  $30720$\\
     MNIST\tablefootnote{http://yann.lecun.com/exdb/mnist/} & $7840$\\
     Extended-MNIST \tablefootnote{https://www.westernsydney.edu.au/bens/home/reproducible\_research/emnist} & $29008$ \\
     STL10\tablefootnote{https://cs.stanford.edu/~acoates/stl10/} & $276480$ \\
    SVHN\tablefootnote{http://ufldl.stanford.edu/housenumbers/}
    & $30720$ \\
    Fashion-MNIST\tablefootnote{https://github.com/zalandoresearch/fashion-mnist} & $7840$
\end{tabular}
\end{center}
\end{minipage}
\end{table}

For Figures \ref{fig:performance} and \ref{fig:N}, each hyper-parameter setting was run till 100 epochs.
In total 10 different seeds were used for each hyper-parameter setting to get the standard error.
The authors had shared access to a computing cluster, consisting of 50 compute nodes with 28 cores each,
which was used to run all the experiments.

\clearpage
\section{Ablation Study}
\label{apx:ablation}
\textbf{Impact of N:} 
In problem setups, where more than $2$ parallel resources are available, a natural question to ask is how well does the proposed method perform then?
To answer this question, we conducted experiments with $N = (2, 4, 6, 8, 10)$ for all of the previous domains.
The results are presented in Figure \ref{fig:N}.
We observe a consistent improvement in performance as the number of available resources increase.
This showcases that while just $2$ parallel gradients sufficed to provide acceleration in high-dimensional problems, our method further benefits when more resources are available.  

\begin{figure*}
    \centering
    \includegraphics[width=0.29\textwidth]{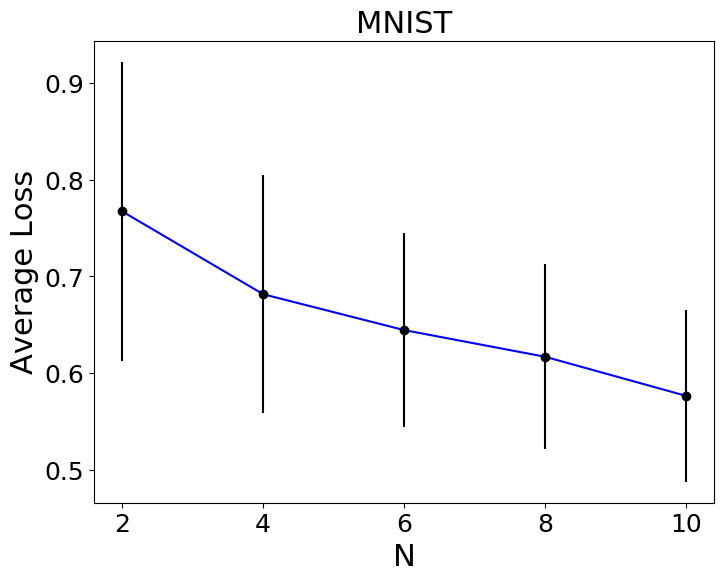} \quad   \includegraphics[width=0.29\textwidth]{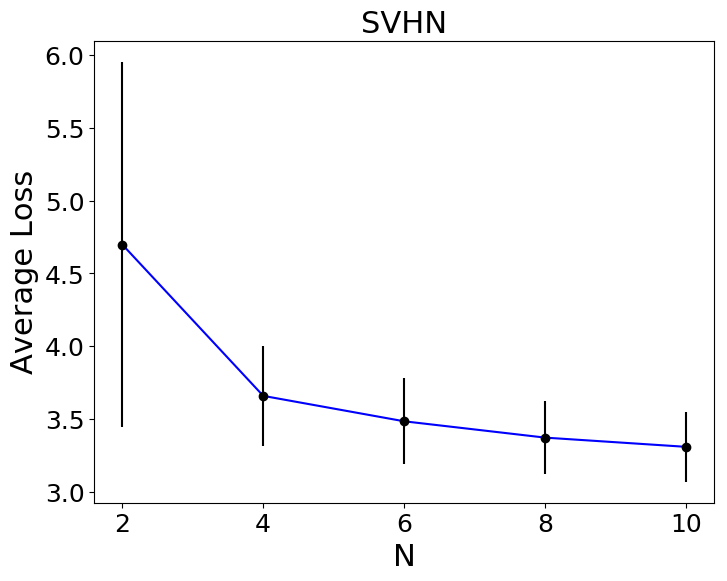} \quad
    \includegraphics[width=0.29\textwidth]{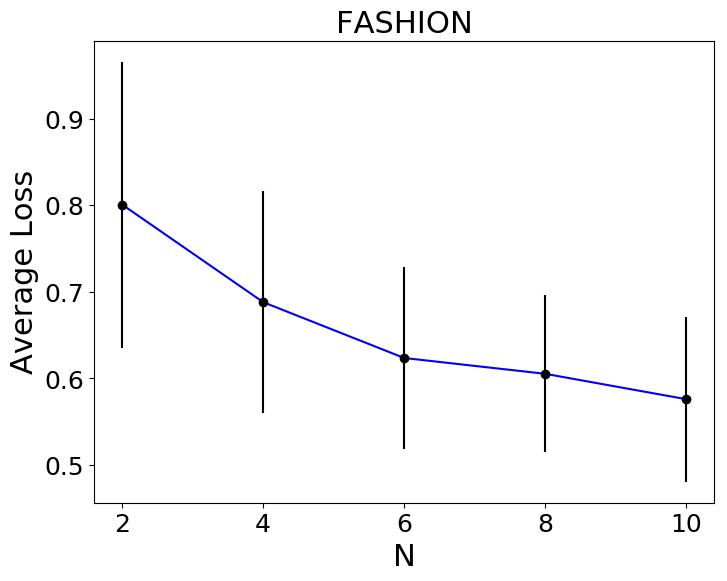}
\\
    \includegraphics[width=0.29\textwidth]{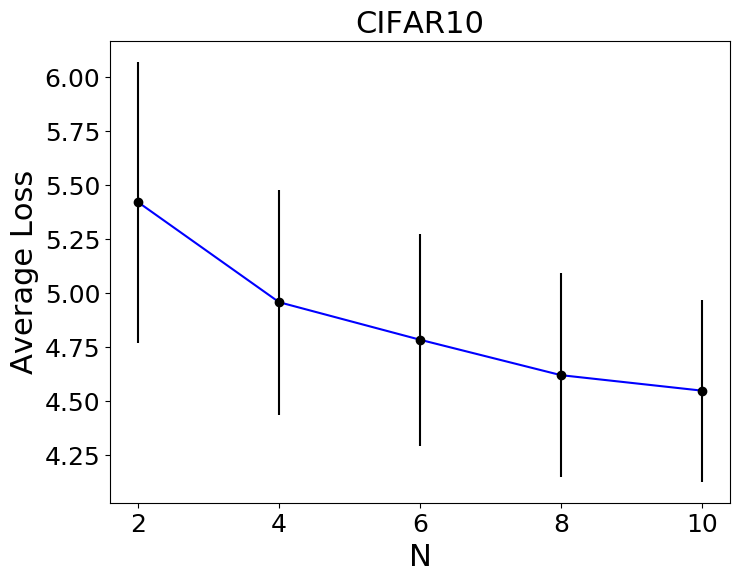} \quad
    \includegraphics[width=0.29\textwidth]{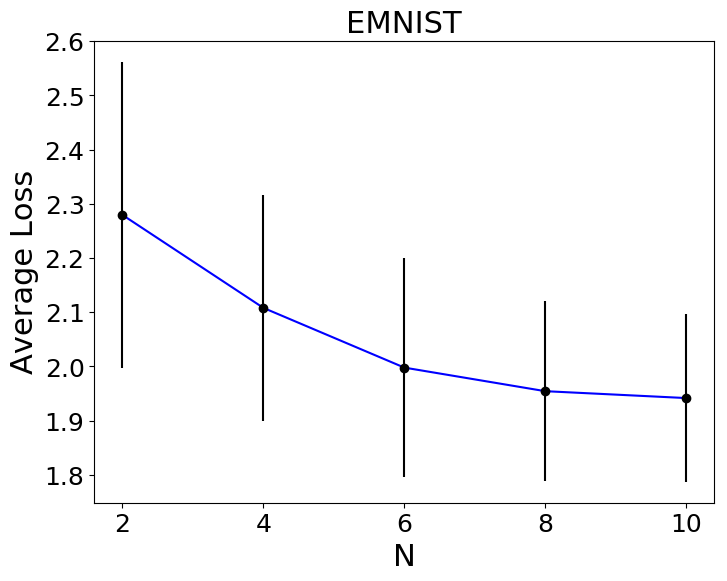} \quad
    \includegraphics[width=0.29\textwidth]{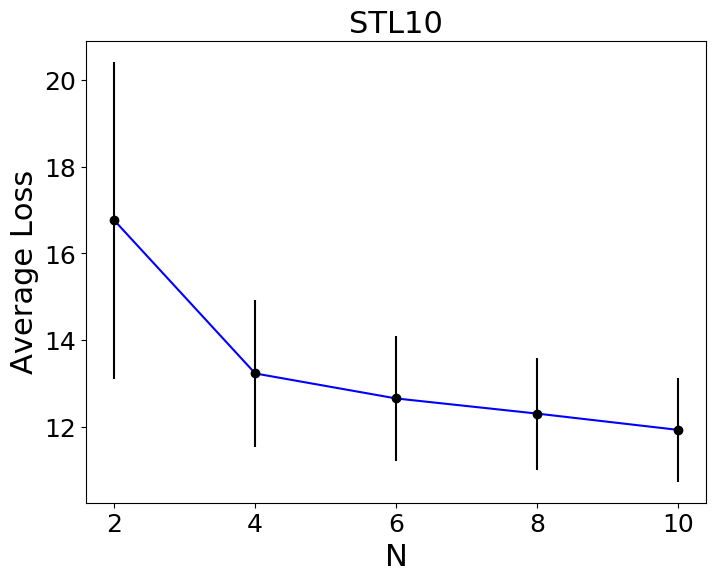}
    \caption{Performance improvement of the proposed Gradient Grouping algorithm for all the problems as the number of available parallel resources increase. 
    The error bars correspond to standard deviation obtained using $10$ trials.}
    \label{fig:N}
    \vspace{-10pt}
\end{figure*}

\end{document}